\documentclass[twoside]{article}
\usepackage{jmlr2e}

\usepackage[parfill]{parskip}

\usepackage{url, hyperref}
\usepackage{algorithm, algpseudocode} %

\usepackage{amsthm, amsmath, amssymb} 
\usepackage[shortlabels]{enumitem}
\usepackage{comment}
\usepackage{color}
\usepackage{multirow}
\usepackage{xspace}
\usepackage{rotating}
\usepackage{graphicx}
\usepackage{caption}
\usepackage{xcolor}
\usepackage[normalem]{ulem}
\usepackage{tikz}

\usepackage{nicefrac}
\usepackage{mathtools}
\mathtoolsset{showonlyrefs,showmanualtags}

\theoremstyle{plain}
\newtheorem{theorem}{\textbf{Theorem}}
\newtheorem{lemma}[theorem]{\textbf{Lemma}}

\newtheorem{proposition}[theorem]{Proposition}

\theoremstyle{definition}

\renewcommand{\cite}{\citep}

\newcommand{\EE}{\mathbb{E}}

\DeclareMathOperator*{\argmin}{arg\,min}

\newcommand{\Vcal}{\mathcal{V}}

\newcommand{\Scal}{\mathcal{S}}

\newcommand{\Acal}{\mathcal{A}}

\newcommand{\Rmax}{R_{\max}}

\newcommand{\Vmax}{V_{\max}}
\newcommand{\RR}{\mathbb{R}}
\newcommand{\Xcal}{\mathcal{X}}

\newcommand{\Dcal}{\mathcal{D}}

\begin{document}
	
\title{A Note on Loss Functions and Error Compounding in Model-based Reinforcement Learning}

\author{
\name Nan Jiang \email nanjiang@illinois.edu
}

\maketitle

\renewcommand*{\theHsection}{\thesection}
\renewcommand*{\theHsubsection}{\thesubsection}

\begin{abstract}
This note clarifies some confusions (and perhaps throws out more) around model-based reinforcement learning and their theoretical understanding  in  the context of deep RL. Main topics of discussion are (1) how to reconcile model-based RL's bad empirical reputation on error compounding with its superior theoretical properties, and (2) the limitations of empirically popular losses. For the latter, concrete counterexamples for the ``MuZero loss'' are constructed to show that it not only fails in stochastic environments, but also suffers exponential sample complexity in deterministic environments when data provides sufficient coverage.  
\end{abstract}


\section{Background}
The relative advantages of model-based and model-free reinforcement learning (RL) are a long-standing debate. Model-free methods learn policies and/or value functions, and the function approximators typically have straightforward label spaces (e.g., action space for policies) with naturally defined loss functions (squared loss for value prediction). 

In contrast, model-based methods approximate the dynamics of the environment, which has a potentially high-dimensional and complex label space, especially in settings commonly found in the deep RL literature (e.g., pixel domains). Theoretical analyses of model-based RL often assume maximum likelihood estimation (MLE), but practical algorithms deviate from this framework in various ways. This disconnection, among other factors, results in a lack of clarity in discussions around model-based RL, and successes and failures are often explained based on reasoning that does not stand scrutiny. For example:

\begin{itemize}
\item Empirically, model-based RL is infamously known for its ``error compounding'' issue, which is anecdotally much worse than model-free methods. However, theory tells us that the error propagation of model-based RL is on par with and sometimes \textit{better} than model-free RL. 
\item When the raw observation (or simply \textit{observation}) is complex and high-dimensional, practical methods often  use a learned encoder to map the observations to a latent state representation (or \textit{latent state}) and fit  latent dynamics. To avoid measuring the observation prediction error, various losses are proposed, such as bisimulation loss and reward prediction loss. However, the validity of these losses in general settings is not widely known, and we provide concrete counterexamples and failure modes of these losses to help understand their limitations and applicability. 
\end{itemize}

This note aims to clarify some of these confusions, with the following caveats:

\begin{enumerate}
\item This note is not intended as a comprehensive survey, and citations can be minimal. Descriptions of common practice are from anecdotes and meant to provide contexts and backgrounds for concrete results and mathematical claims, which the readers should interpret by themselves. 
\item When it comes to a family of methods, the note will often use a highly simplified and minimal setup and cannot include all the practical variants. That said, the high-level claims will likely hold for small variants of the same idea. 
\end{enumerate}

\section{Setup and Notation}
\paragraph{Markov Decision Processes (MDPs)} We assume standard notation for an infinite-horizon discounted MDP $M^\star = (\Scal, \Acal, P^\star, R^\star, \gamma, s_{\textrm{init}})$ as the true environment, where $s_{\textrm{init}}$ is the deterministic initial state  w.l.o.g. All rewards are assumed to be bounded in $[0, \Rmax]$, and  $\Vmax:= \Rmax/(1-\gamma)$ is an upper bound on the cumulative reward and value functions. We assume $\Scal$ is large and thus tabular methods do not apply. In most cases we consider  finite and small action spaces, but some claims also hold in large action spaces. 
Given a policy $\pi$, its expected return is $J_{M^\star}(\pi) := \EE_{M^\star, \pi}[\sum_{t=0}^\infty \gamma^t r_t]$. Sometimes, especially in counterexamples, it will be easier to consider an $H$-step finite-horizon problem, where the expected return is defined as (with an abuse of notation) $J_{M^\star}(\pi):=\EE_{\pi}[\sum_{h=1}^H r_h]$. 

\paragraph{Learning Settings} For the discounted setting, a common simplification is to assume that  the dataset $\Dcal$ consists of i.i.d.~tuples $\{(s,a,r,s')\}$, where $(s,a) \sim d^{\Dcal}$, $r = R^\star(s,a)$, $s' \sim P^\star(\cdot|s,a)$. For the finite-horizon setting, the dataset consists of i.i.d.~trajectories $(s_1, a_1, r_1, \ldots, s_H, a_H, r_H)$ sampled from some behavior policy $\pi_\Dcal$. 

\paragraph{Planning and Off-Policy Evaluation} We consider algorithms that learn a model $M$ from data. Note that here $M$ may not have the same type as $M^\star$ as it can be a latent-state model (Section~\ref{sec:latent}). However, regardless of its inner workings, a model should be able to perform (1) planning: outputting a policy $\pi_{M}^\star$ that is hopefully near-optimal, and (2) off-policy evaluation (OPE): estimating $J_{M^\star}(\pi)$ for a given
$\pi$. We use $J_M(\pi)$ to denote the estimated value given by model $M$. 
We will examine various algorithms and see if they correctly perform these tasks with a sufficient amount of data and appropriate coverage assumptions. We will focus on OPE most of the times, as it avoids the pessimism/optimism aspects of planning which is orthogonal to the main topics. 

\section{Error Compounding?} \label{sec:compound}
We get to the first question raised in the introduction: how to reconcile the bad reputation on error compounding of model-based RL with its superior theoretical properties? To be more concrete, below is a standard result that characterizes the error propagation in model-based RL. For simplicity, in this section we will only consider candidate model that makes raw observation predictions and the reward function is known. 

\begin{lemma}[Simulation Lemma]
For any $P: \Scal\times \Acal\to \Delta(\Scal)$ and any $\pi: \Scal\to \Delta(\Acal)$, let $J_M(\pi)$ be the expected return of $\pi$ in $M$ specified by $(P, R^\star)$. Then,
\begin{align} \label{eq:simu}
|J_{M^\star}(\pi) - J_M(\pi)| = &~ \left|\frac{\gamma}{1-\gamma} \, \EE_{(s,a) \sim d_{M^\star}^\pi}[\langle P(\cdot|s,a) - P^\star(\cdot|s,a), V_{M}^\pi(\cdot) \rangle ] \right| \\
\le &~ \frac{\Vmax}{1-\gamma} \, \EE_{(s,a) \sim d_{M^\star}^\pi}[\|P(\cdot|s,a) - P^\star(\cdot|s,a)\|_1]. \label{eq:tv}
\end{align}
where $d_{M^\star}^\pi$ is the normalized discounted state-action occupancy induced from $\pi$ in $M^\star$. 
\end{lemma}
The $1/(1-\gamma)$ term clearly indicates a linear-in-horizon error accumulation, which is \textbf{the best one can hope for} in general: almost all algorithms that learn from 1-step transition data, including many model-free algorithms (especially in the TD family), have at least linear-in-horizon error propagation. The above lemma is for OPE; when it comes to planning, popular model-free methods often incur \textit{quadratic} dependence on horizon (i.e., $1/(1-\gamma)^2$) \citep{scherrer2012use}, whereas model-based method still enjoys linear error accumulation. 

So how to reconcile this result with the practical claim that error compounding is particularly bad for model-based RL? Here we give one possible explanation, and for that we shall turn to the 1-step error term: the lemma shows that model-based RL performs OPE accurately, if $\|P(\cdot|s,a) - P^\star(\cdot|s,a)\|_1$ can be made small under the distribution of $d_{M^\star}^\pi$. However, $\|\cdot\|_1$ corresponds to total-variation error and is hard to estimate/minimize directly (with exceptions \cite{sun2019model}). Instead, the learning algorithms (in theory) perform MLE:
$$
\textstyle \argmin_{P \in \mathcal{P}} \sum_{(s,a,s') \in \Dcal} - \log P(s'|s,a). 
$$
With   realizability assumptions and sufficient data, one can show that the KL-divergence between $P(\cdot | s,a)$ and $P^\star(\cdot|s,a)$ is small, which leads to an upper bound on TV error via Pinsker's inequality.\footnote{For any two distributions $p$ and $q$ over the same space, Pinsker's inequality states that $\|p-q\|_{\textrm{TV}} \le \sqrt{\tfrac{1}{2}\textrm{KL}(p \,\|\, q)}$.} Finally, this error is averaged under the data distribution $d^\Dcal$, but can be translated to $d_{M^\star}^\pi$ if data covers $\pi$, e.g., $\|d_{M^\star}^\pi / d^\Dcal\|_\infty$ is bounded.


\paragraph{Deterministic models and $L_2$ loss} The reason that the above theory often does not apply is that, in the deep RL literature, very few works use MLE loss with a stochastic model. A potential reason is that MLE treats the problem as a classification task with a huge and unstructured label space, where each possible state is a separate label, so it can be hard to work with. Instead of using the MLE loss with a stochastic model, what is more common is the following setup:
\begin{itemize}
\item \textbf{Deterministic model.} The model makes deterministic predictions of the next state, i.e., $\hat s' = f(s,a)$.
\item \textbf{Non-MLE loss.} Loss is $\|s' - f(s,a) \|$ (perhaps squared). Here the common setup is that states are embedded in an Euclidean space (e.g., $\Scal \subset \RR^d$), and $\|\cdot\|$ is the $L_2$ norm. An example is that $s$ is pixel observations and $\|\cdot\|$ is the $L_2$ norm, but there are other choices (e.g., rescaling the axes of $\RR^d$   due to normalization leads to a different $L_2$ loss). 
\end{itemize}

Of course, there is the question of whether the true environment can be well-approximated by deterministic dynamics. But even that aside, the above formulation \textbf{simply does not lead to a small TV error} without further assumptions. 
There is some empirical evidence supporting this view, related to whether model predicts ``legal'' vs.~``illegal'' states: \citet{talvitie2017self} showed that learned models (over raw pixels) can easily generate illegal configurations that can never be produced by the real system, and often   this leads to bad error accumulation. In contrast, if TV error were properly minimized, the learned model should generate real observations most of the times. 


\paragraph{Smoothness} The above argument explains why model-based RL can have bad error-compounding behaviors, and suggests that it is not due to error accumulation but likely the 1-step error itself being incompatible with what theory requires. That said, we still need to reconcile with the fact that model-based RL \textit{does} work sometimes, even in domains with high-dimensional observations, and answer the question: when is $\|s' - f(s,a)\|$ a justified loss? 

A plausible explanation is smoothness. For now let's suppose $P^\star$ is also deterministic and we write it as $s' = f^\star(s,a)$. The question is when a small $\|f^\star(s,a) - f(s,a)\|$, which is the loss function, translates to small $|J_{M}(\pi) - J_{M^\star}(\pi)|$ ($M$ is the model with $f$ as the dynamics and $R^\star$ as the reward function). It turns out that a sufficient condition is $V_M^{\pi}$ being smooth in $\|\cdot\|$. 

In particular, if $\forall s,\tilde s \in \Scal$,
\begin{align*}
& \textrm{(Lipschitz smoothness of $V_M^\pi$ in $\|\cdot\|$)} & \qquad  |V_M^\pi(s) - V_M^\pi(\tilde s)| \le L \|s - \tilde s\|,
\end{align*}
then we immediately have
$$
|V_M^\pi(f(s,a)) - V_{M}^\pi(f^\star(s,a))|  \le L \|f(s,a) - f^\star(s,a)\|.
$$
The LHS is $|\langle P(\cdot|s,a) - P^\star(\cdot|s,a), V_{M}^\pi(\cdot) \rangle|$ from Eq.\eqref{eq:simu}, or more precisely, its special case when both $P$ and $P^\star$ are deterministic (which we denote as $f$ and $f^\star$, respectively). This is a tighter measure of model error than the TV error and is directly controlled by 
$\|f(s,a) - f^\star(s,a)\|$. 
This analysis can also be extended to the stochastic setting, where one should minimize the Wasserstein distance between $P^\star(\cdot|s,a)$ and $P(\cdot|s,a)$ with $\|\cdot\|$ as the base metric \citep{asadi2018lipschitz}.

\paragraph{Linear control problems} A concrete problem setting where algorithms work both in practice and in theory is linear control. In a standard LQR model, not only using the $\|s' - f(s,a)\|^2$ loss corresponds to MLE (assuming a known Gaussian noise model), the value functions are also always quadratic in the state vector, which naturally yields smoothness. 

\section{Latent Models and Non-observation-prediction Losses} \label{sec:latent}

\newcommand{\lat}{\textrm{latent}}

Many difficulties discussed in the previous section arise from the fact that the problem of predicting the next raw observation has a very large label space. To circumvent this, a popular approach is to learn \textit{latent} dynamics. More precisely, a latent model consists of (at least) two components:
\begin{enumerate}
\item An encoder $\phi: \Scal\to\Xcal$ that maps raw observations to some latent state space $\Xcal$. $\Xcal$ is generally considered ``nicer'' than $\Scal$ in some way (smaller in the discrete case, lower-dimensional or less entangled in the continuous case, etc.).
\item Latent dynamics $P_\lat: \Xcal\times\Acal\to\Delta(\Xcal)$. 
\end{enumerate}
For simplicity we assume that reward is known, and can be expressed by any candidate encoder $\phi$ under consideration. That is, $R(s,a)$ only depends on $s$ through $\phi(s)$. Given such a model,\footnote{OPE is similar, except that the latent model can evaluate policies that depend on $s$ through $\phi(s)$.} planning can be done by (1) planning in the latent model, which is an MDP defined over $\Xcal$ as its state space, and (2) \textit{lifting} the policy to the observation space: given $\pi_{\lat}: \Xcal\to\Delta(\Acal)$, the lifted policy is $s \mapsto \pi_{\lat}(\phi(s))$. 

The remaining question is: what loss should we use to learn $(\phi, P_\lat)$? 
As a starter, we can also learn a decoder $\psi: \Xcal \to \Delta(\Scal)$, in which case we can still use the losses in Section~\ref{sec:compound}. However, in this case $(\phi, P_{\lat}, \psi)$ together is just a particular parameterization of a $\Scal\times\Acal\to\Delta(\Scal)$ model and suffers the same issues. Below we are  interested in losses that do not require a decoder. 

\subsection{Bisimulation loss}
One natural idea is to check if the latent state ``rolls forward'' on its own. This idea finds its root in the bisimulation literature \citep{ferns2004metrics}, which is easier to describe when $\Scal$ and $\Xcal$ are discrete: define
$$ \textstyle
P_\phi^\star(x'|s,a) := \sum_{s': \phi(s') = x'} P^\star(s'|s,a).
$$
$\phi^\star$ is a bisimulation abstraction with $(R_{\lat}^\star, P_{\lat}^\star)$ as the induced abstract model, if for all $s \in \Scal, a\in\Acal$, $x' \in \Xcal$,
\begin{align*}
R^\star(s,a) = R_{\lat}^\star(\phi^\star(s),a), \qquad
P_{\phi^\star}^\star(x'|s,a) =   P_{\lat}^\star(x'|\phi^\star(s), a).
\end{align*}
Such a model is guaranteed to recover the optimal policy and   accurately evaluate policies that depend on $\phi^\star(s)$ \citep{nan_abstraction_notes}. Bisimulation is an attractive conceptual framework for focusing on task-relevant aspects of the environment and ignoring distractions (c.f.~the ``noisy-TV'' problem). 

\textbf{Loss and its problems in stochastic environments}~~
Learning a good $P_{\lat}$ is easy if a good $\phi$ is given and fixed: all we need to do is (1) convert $\Dcal$ into a dataset over $\Xcal$, by mapping each transition tuple $(s,a,r,s')$ to $(x,a,r,x'):= (\phi(s), a, r, \phi(s'))$, and (2) learn a $P_{\lat}$ that predicts $x'$ well from $(x,a)$ using MLE or other losses, as discussed in Section~\ref{sec:compound}. To keep things simple, let's stick to MLE for now and assume it is effective on the $\Xcal$ space. 

The tricky part is to design a loss when $\phi$ is not given and must be learned. A natural idea is
$$
\argmin_{\phi, \, P_{\lat}} \, \ell(P_{\lat}; \phi(\Dcal)).
$$
Here $\phi(\Dcal)$ is the converted dataset $(x,a,r,x')$ based on a given $\phi$, and $\ell$ is the MLE loss of predicting $x'=\phi(s')$ from $x=\phi(s)$ and $a$, i.e.,  $\sum_{(s,a,s')} - P_{\lat}(\phi(s')|\phi(s),a)$. Recall that we assume all candidate $\phi$ can perfectly predict rewards, otherwise a reward error term should   be included. 

This approach likely works in theory if ``true'' latent dynamics (in the sense of $P_{\lat}^\star$ induced by some bisimulation $\phi^\star$) are deterministic and we can minimize the loss $\ell$ to nearly $0$. 
However, if the true latent dynamics are stochastic, the loss runs into a variant of the infamous \textit{double-sampling problem} for value-based RL \citep{baird1995residual, farahmand2011model}. 
At an intuitive level, the problem is that by choosing $\phi$ we are not only doing feature extraction in the input space ($x = \phi(s)$), but at the same time defining what labels we need to predict ($x' = \phi(s')$), creating an incentive for choosing degenerate $\phi$ that makes $x' = \phi(s')$ easier to predict. 

\newcommand{\entropy}{\textrm{entropy}}

More concretely: fix an $(s,a)$ pair and let $(\phi^\star, P_{\lat}^\star)$ be a perfect bisimulation abstraction and its induced model. The expected MLE loss (w.r.t.~the randomness of $s' \sim P^\star(\cdot|s,a)$) of $(\phi^\star, P_{\lat}^\star)$ on this $(s,a)$ is (the overall loss takes a further expectation over $(s,a) \sim d^\Dcal$) 
$$-\sum_{x'\in \Xcal} P_{\lat}^\star(x'|\phi^\star(s), a) \log P_{\lat}^\star(x'|\phi^\star(s), a) = \entropy(P_{\lat}^\star(\cdot|\phi^\star(s), a)).$$ 
This value can be non-zero, but that is ok---if the dynamics are inherently stochastic, no one can make perfect predictions and achieve $0$ loss. All we need is predictions that are \textit{as good as} predicting $x'$ from $(s, a)$. In other words, what we really want to minimize to $0$ is a form of \textbf{excess risk}: 
$$
\entropy(P_{\lat}(\cdot|\phi(s), a)) - \entropy(P_\phi^\star(\cdot|s, a)).
$$
The problem is, we do not have access to the $\entropy([P_\phi^\star(\cdot|s, a)])$ term, and \textit{its value changes with $\phi$}. So when we see a $(\phi, P_{\lat})$ pair with a high loss, we cannot tell between the two situations:
\begin{enumerate}
\item $\entropy(P_\phi^\star(\cdot|s, a))$ is low and excess risk high --- bad $(\phi, P_{\lat})$.
\item $\entropy(P_\phi^\star(\cdot|s, a))$ is high and excess risk low --- good $(\phi, P_{\lat})$. 
\end{enumerate}


\subsection{Multi-step reward prediction loss}
The issue with bisimulation loss makes it desirable that we ask all models to predict the same object in the loss. The multi-step reward prediction loss (also colloquially known as the ``MuZero loss'' \cite{schrittwieser2020mastering}), precisely does this. 

For clarity, it will be convenient to switch to the finite-horizon setting, where a data trajectory looks like $s_1, a_1, r_1, \ldots, s_H, a_H, r_H$.\footnote{This avoids the question of ``how many steps should we roll out'', as in the finite-horizon setting it is natural to roll all the way to the end.} The loss is defined as follows: given a candidate model $P: \Scal\times\Acal\to\Delta(\Scal)$ (we will comment on latent models later), for each $h$, and for each data trajectory, let $s_h$ be the state observed in the data. Then, starting from $s_h$, let $\hat s_h = s_h$, and for $h'=h, h+1, \ldots, H$, 
\begin{itemize}
\item Predict reward: $\hat r_{h'} = R(\hat s_{h'}, a)$.
\item Sample the next state using the candidate model $P$: $\hat s_{h'+1} \sim P(\cdot| \hat s_{h'}, a_{h'})$. 
\end{itemize}
The loss associated with $s_h$ is $\sum_{h'=h}^H (r_h - \hat r_h)^2$, and then we sum  over all $h$ and average over trajectories. 
Essentially we start from $s_h$ in the data, reproduce the data actions $a_{h:H}$, but use the candidate dynamics $P$ to roll-out states. However, we do not measure whether the sampled states are ``correct'', but only use these states to predict the rewards $\hat r_h$ and measure reward prediction error. This way, all models ``speak the same language'' and seem to be compared fairly. 

The above description assumes an observation-level dynamics, but one can easily apply it to latent models as well. As we will explain, however, to see the problems with the loss we do not even need to consider latent models, so we will restrict ourselves to observation-level models in discussion. 

\paragraph{True model does not minimize the loss in stochastic environments} Reproducing data actions $a_{h:H}$ without looking at the states $s_{h+1:H}$ is inherently treating the action sequence in an ``open-loop'' manner, which leads to problems in stochastic environments, as also pointed out by \citet{voelcker2023lambda}. 

\begin{figure}[t]
\centering
\includegraphics[scale=.8]{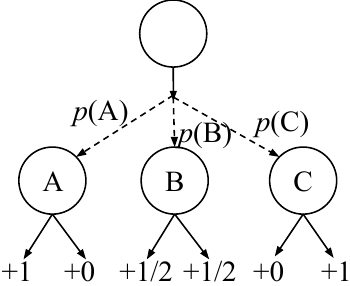}
\caption{Construction for Proposition~\ref{prop:muzero_stoch}. The top state is the initial state, where both actions lead to the same distribution over A, B, C. The numbers at the second level indicate rewards for taking left (L) and right (R) actions from states A, B, and C, respectively. \label{fig:muzero_stoch}}
\end{figure}

\begin{proposition} \label{prop:muzero_stoch}
There exist stochastic finite-horizon MDPs and a data collection policy, where (1) some $M\ne M^\star$ has smaller loss than $M^\star$, even with infinite amount of data, and (2) $M$ gives wrong evaluation to some policies.
\end{proposition}
\begin{proof}
See construction in Figure~\ref{fig:muzero_stoch}, which is a finite-horizon MDP with two actions (L and R) and $H=2$. Let data collection policy be uniformly random. We will consider two models, $M$ and $M^\star$, and both of them agree on the end rewards and only differ in $P(A), P(B), P(C)$, so the reward prediction loss at $h=2$ is $0$ and we only need to look at loss at $h=1$. The true model $M^\star$ has $P(A)=P(C)=0.5$ and $P(B) = 0$, and the wrong model $M$ has $P(B) = 1$. The expected loss for $M^\star$ is $1/2$, but that for $M$ is $1/4$. Furthermore, $\pi = \pi_{M^\star}^\star$ (i.e., taking L in A and R in C) is evaluated differently in $M^\star$ ($J_{M^\star}(\pi) = 1$) and in $M$ ($J_M(\pi) = 1/2$). 
\end{proof}

One subtlety is that state $B$ never appears in data, which seems to indicate insufficient data coverage. This can be fixed by making $P(B)$ non-zero in $M^\star$. In fact, the standard definition of coverage (as in the analyses sketched in Section~\ref{sec:compound}) considers the boundedness of $d_{M^\star}^\pi(s,a) / d_{M^\star}^{\pi_\Dcal}(s,a)$, which is well bounded even when $P(B) = 0$. 

\paragraph{Exponential distribution shift} Suppose we are happy with deterministic dynamics. Does the reward-prediction loss enjoy good guarantees? Asymptotically, yes. If we take actions uniformly at random, we will see every possible action sequence with non-zero probability, and guarantee that the reward prediction for the action sequence is accurate, if we have infinite amount of data. 

However, the situation is less favorable when we examine its finite-sample behavior. In OPE, finite-sample errors crucially depend on the notion of \textit{coverage}, which can be viewed as a measure of ``effective sample size'' that characterizes how much information in the training data is relevant to evaluating the target policy. 
The technical definition of coverage varies across OPE algorithms: some algorithms, like importance sampling 
\citep{precup2000eligibility}, require \textit{trajectory coverage}, measured by the boundedness of $\prod_{h=1}^H \frac{\pi(a_h|s_h)}{\pi_\Dcal(a_h|s_h)}$, which can be   exponential if $\pi_\Dcal$ are $\pi$ are not close. As briefly discussed several times, model-based RL with MLE loss only requires \textit{state-action coverage}, often measured by the boundedness of $d_{M^\star}^{\pi}(s_h,a_h) / d_{M^\star}^{\pi_\Dcal}(s_h,a_h)$, which can be much smaller than the cumulative product of importance weights \cite{munos2007performance}. 

Below we show that, the reward prediction loss cannot have a guarantee that depends on state-action coverage, and trajectory coverage is a more accurate characterization. 

\begin{proposition} \label{prop:muzero_exp}
There exists a deterministic MDP $M^\star$, a data-collection policy $\pi_\Dcal$,  and a target policy $\pi$, such that (1) $d_{M^\star}^\pi(s_h,a_h) / d^\Dcal(s_h,a_h) \le 2$ for any $s_h, a_h$, and (2) there exists a model $M$ that makes wrong OPE predictions, yet the probability that reward-prediction loss can distinguish it from $M^\star$ is exponentially small in horizon using a polynomial-sized dataset. 
\end{proposition}
\begin{proof}
Let the state space at level $h$ be \{L,R\}$^{h-1}$ and there are two actions, L and R. We first describe the wrong model $M$: any $s_h$, $a_h$ deterministically transitions to a next-state, which simply appends $a_h$ to $s_h$. This is essentially a complete-tree MDP. There are no intermediate rewards. At $h=H$, all sequences ending with L have reward $1$. Those ending with R have reward $0$, with the exception that the reward for the all-R sequence is 100.

We now describe the true model $M^\star$: any $s_h$, $a_h$ deterministically transitions to $(L, \ldots, L)$. No intermediate rewards. At $h=H$, all sequences ending with L have reward $1$, and those ending with R have reward $0$. This is equivalent to a chain MDP where each level only has one state. Let $\pi_\Dcal$ be uniformly random, and it is easy to see that the data provides a constant coverage for all policies. 

To verify the claims, first notice that $M$ and $M^\star$ always give the same multi-step reward prediction for any action sequence, unless an all-R action sequence is taken from $h=1$, which happens with $(1/2)^H$ probability under $\pi_\Dcal$. However, the all-R policy is evaluated to give a return of $100$ in $M$, which is different from its true return $0$ in $M^\star$. 
\end{proof}

Some readers may complain that the data fails to cover most part of the tree, so it is more of a problem with the definition of coverage  than the loss. This is reasonable. That said, (1) MLE works under this coverage assumption and enjoys polynomial sample complexities, and (2) if $M$ is allowed to have its own latent state space, we can make the true environment a single chain and let the latent state of the model record the history. This way all real states will be properly covered. So this example is also related to the discussion of ``illegal states'' in Section~\ref{sec:compound}. 

\subsection{Other losses}
There are other losses that this note does not cover. Some of them have been studied carefully in theory and thus have their properties and limitations better understood. Notable examples include the value-aware loss \cite{farahmand2017value, sun2019model} ($\sup_{V\in\Vcal} \langle P(\cdot|s,a) - P^\star(\cdot|s,a), V\rangle$), inverse kinematics loss \citep[predicting $a_h$ from $\phi(s_h)$ and $\phi(s_{h+1})$:][]{pathak2017curiosity, mhammedi2023representation}, and a contrastive loss for telling real transitions from fake ones \citep{misra2020kinematic}. 

\section{Discussion}
Readers might notice phrases like ``in general'' and ``without further assumptions'' when this note makes a negative claim. Indeed, the counterexamples are meant to not only help build a more accurate mental model of how these methods work and what their limitations are, but also guide the search of additional assumptions and problem structures that circumvent these difficulties. As an example, Proposition~\ref{prop:muzero_stoch} shows that the multi-step reward prediction loss is problematic for stochastic environments. However, this does not mean that it cannot work in more structured settings; recently, \citet{tian2023toward} showed that the loss may be actually justified in LQG systems. It will be interesting to see whether such positive results can be extended to more general settings.

Along similar lines, writing this note leaves me the impression that properties of the \textit{learned} model (instead of the true environment) play a very important role in model-based RL, but have received disproportionately little attention. For example, the smoothness discussed in Section~\ref{sec:compound} is about the value function in $M$, not $M^\star$. On the other hand, the notion of coverage used by current analyses concerns the ratio $d_{M^\star}^{\pi} / d^\Dcal$, but if we invoke the simulation lemma in the other direction we obtain $d_{M}^\pi / d^\Dcal$, and requiring this alternative coverage will break the counterexample in Proposition~\ref{prop:muzero_exp}. And so on. An advantage of working with properties of $M$ is that, unlike the unverifiable properties of $M^\star$, $M$ is known to the learner, and understanding of its desirable 
properties may directly translate to algorithmic insights such as regularization. 

\section*{Acknowledgements}
The note was inspired by a conversation with Joshua  Zitovsky on model-based algorithms in deep RL. 
The author thanks Akshay Krishnamurthy and Ching-An Cheng for detailed feedback on a draft version of the note, and Wen Sun, Yi Tian, Kaiqing Zhang for valuable discussions. 

\bibliography{RL}

\begin{thebibliography}{17}
\providecommand{\natexlab}[1]{#1}
\providecommand{\url}[1]{\texttt{#1}}
\expandafter\ifx\csname urlstyle\endcsname\relax
  \providecommand{\doi}[1]{doi: #1}\else
  \providecommand{\doi}{doi: \begingroup \urlstyle{rm}\Url}\fi

\bibitem[Asadi et~al.(2018)Asadi, Misra, and Littman]{asadi2018lipschitz}
Kavosh Asadi, Dipendra Misra, and Michael Littman.
\newblock Lipschitz continuity in model-based reinforcement learning.
\newblock In \emph{International Conference on Machine Learning}, pages 264--273. PMLR, 2018.

\bibitem[Baird(1995)]{baird1995residual}
Leemon Baird.
\newblock Residual algorithms: Reinforcement learning with function approximation.
\newblock In \emph{Machine Learning Proceedings 1995}, pages 30--37. Elsevier, 1995.

\bibitem[Farahmand and Szepesv{\'a}ri(2011)]{farahmand2011model}
Amir-massoud Farahmand and Csaba Szepesv{\'a}ri.
\newblock Model selection in reinforcement learning.
\newblock \emph{Machine learning}, 85\penalty0 (3):\penalty0 299--332, 2011.

\bibitem[Farahmand et~al.(2017)Farahmand, Barreto, and Nikovski]{farahmand2017value}
Amir-massoud Farahmand, Andre Barreto, and Daniel Nikovski.
\newblock Value-aware loss function for model-based reinforcement learning.
\newblock In \emph{Artificial Intelligence and Statistics}, 2017.

\bibitem[Ferns et~al.(2004)Ferns, Panangaden, and Precup]{ferns2004metrics}
Norm Ferns, Prakash Panangaden, and Doina Precup.
\newblock Metrics for finite {M}arkov decision processes.
\newblock In \emph{Proceedings of Uncertainty in Artificial Intelligence}, pages 162--169, 2004.

\bibitem[Jiang(2018)]{nan_abstraction_notes}
Nan Jiang.
\newblock \emph{{Notes on State Abstractions}}.
\newblock {University of Illinois at Urbana-Champaign}, 2018.
\newblock \url{http://nanjiang.cs.illinois.edu/files/cs598/note4.pdf}.

\bibitem[Mhammedi et~al.(2023)Mhammedi, Foster, and Rakhlin]{mhammedi2023representation}
Zakaria Mhammedi, Dylan~J Foster, and Alexander Rakhlin.
\newblock Representation learning with multi-step inverse kinematics: An efficient and optimal approach to rich-observation rl.
\newblock In \emph{International Conference on Machine Learning}, pages 24659--24700. PMLR, 2023.

\bibitem[Misra et~al.(2020)Misra, Henaff, Krishnamurthy, and Langford]{misra2020kinematic}
Dipendra Misra, Mikael Henaff, Akshay Krishnamurthy, and John Langford.
\newblock Kinematic state abstraction and provably efficient rich-observation reinforcement learning.
\newblock In \emph{International conference on machine learning}, pages 6961--6971. PMLR, 2020.

\bibitem[Munos(2007)]{munos2007performance}
R{\'e}mi Munos.
\newblock Performance bounds in l\_p-norm for approximate value iteration.
\newblock \emph{SIAM journal on control and optimization}, 46\penalty0 (2):\penalty0 541--561, 2007.

\bibitem[Pathak et~al.(2017)Pathak, Agrawal, Efros, and Darrell]{pathak2017curiosity}
Deepak Pathak, Pulkit Agrawal, Alexei~A Efros, and Trevor Darrell.
\newblock Curiosity-driven exploration by self-supervised prediction.
\newblock In \emph{International conference on machine learning}, pages 2778--2787. PMLR, 2017.

\bibitem[Precup et~al.(2000)Precup, Sutton, and Singh]{precup2000eligibility}
Doina Precup, Richard~S Sutton, and Satinder~P Singh.
\newblock Eligibility traces for off-policy policy evaluation.
\newblock In \emph{Proceedings of the Seventeenth International Conference on Machine Learning}, pages 759--766, 2000.

\bibitem[Scherrer and Lesner(2012)]{scherrer2012use}
Bruno Scherrer and Boris Lesner.
\newblock On the use of non-stationary policies for stationary infinite-horizon markov decision processes.
\newblock \emph{Advances in Neural Information Processing Systems}, 25, 2012.

\bibitem[Schrittwieser et~al.(2020)Schrittwieser, Antonoglou, Hubert, Simonyan, Sifre, Schmitt, Guez, Lockhart, Hassabis, Graepel, et~al.]{schrittwieser2020mastering}
Julian Schrittwieser, Ioannis Antonoglou, Thomas Hubert, Karen Simonyan, Laurent Sifre, Simon Schmitt, Arthur Guez, Edward Lockhart, Demis Hassabis, Thore Graepel, et~al.
\newblock Mastering atari, go, chess and shogi by planning with a learned model.
\newblock \emph{Nature}, 588\penalty0 (7839):\penalty0 604--609, 2020.

\bibitem[Sun et~al.(2019)Sun, Jiang, Krishnamurthy, Agarwal, and Langford]{sun2019model}
Wen Sun, Nan Jiang, Akshay Krishnamurthy, Alekh Agarwal, and John Langford.
\newblock {Model-based RL in Contextual Decision Processes: PAC bounds and Exponential Improvements over Model-free Approaches}.
\newblock In \emph{Conference on Learning Theory}, 2019.

\bibitem[Talvitie(2017)]{talvitie2017self}
Erik Talvitie.
\newblock Self-correcting models for model-based reinforcement learning.
\newblock In \emph{AAAI Conference on Artificial Intelligence}, 2017.

\bibitem[Tian et~al.(2023)Tian, Zhang, Tedrake, and Sra]{tian2023toward}
Yi~Tian, Kaiqing Zhang, Russ Tedrake, and Suvrit Sra.
\newblock Toward understanding state representation learning in muzero: A case study in linear quadratic gaussian control.
\newblock In \emph{2023 62nd IEEE Conference on Decision and Control (CDC)}, pages 6166--6171. IEEE, 2023.

\bibitem[Voelcker et~al.(2023)Voelcker, Ahmadian, Abachi, Gilitschenski, and Farahmand]{voelcker2023lambda}
Claas~A Voelcker, Arash Ahmadian, Romina Abachi, Igor Gilitschenski, and Amir-massoud Farahmand.
\newblock {$\lambda$-AC: Effective decision-aware reinforcement learning with latent models}.
\newblock 2023.

\end{thebibliography}
\end{document}